\documentclass{article}

\usepackage{arxiv}

\usepackage[utf8]{inputenc} 
\usepackage[T1]{fontenc}    
\usepackage{hyperref}       
\usepackage{url}            
\usepackage{booktabs}       
\usepackage{amsfonts}       
\usepackage{nicefrac}       
\usepackage{microtype}      
\usepackage{lipsum}		
\usepackage{graphicx}
\usepackage{biblatex}
\usepackage{doi}

\usepackage{times}  
\usepackage{helvet}  
\usepackage{courier}  

\usepackage{algorithm}

\usepackage{mwe}
\usepackage{caption}
\usepackage{subcaption}
\usepackage{color}
\usepackage{amsmath}
\usepackage[noend]{algpseudocode}
\usepackage{amssymb}
\usepackage{mathtools}
\usepackage{dsfont}
\DeclarePairedDelimiter\ceil{\lceil}{\rceil}

\newtheorem{lemma}{Lemma}
\newtheorem{theorem}{Theorem}
\newtheorem{proposition}{Proposition}
\newtheorem{proof}{Proof}

\newtheorem{assumption}{Assumption}

\usepackage{newfloat}
\usepackage{listings}
\floatstyle{ruled}
\newfloat{listing}{tb}{lst}{}
\floatname{listing}{Listing}

\title{Finite Horizon Q-learning: Stability, Convergence, Simulations and an Application on Smart Grids}

\author{ Vivek VP\\
	Department of Computer Science and Automation\\
	Indian Institute of Science\\
	Bangalore, 560012 \\
	\texttt{vivekv@iisc.ac.in} \\
	\And
	Dr Shalabh Bhatnagar \\
	Department of Computer Science and Automation\\
	Indian Institute of Science\\
	Bangalore, 560012 \\
	\texttt{shalabh@iisc.ac.in} \\
}

\renewcommand{\shorttitle}{Finite Horizon Q-Learning}

\hypersetup{
pdftitle={A template for the arxiv style},
pdfsubject={q-bio.NC, q-bio.QM},
pdfauthor={Vivek VP, Shalabh Bhatnagar},
pdfkeywords={First keyword, Second keyword, More},
}
\begin{document}

\maketitle

\begin{abstract}
Q-learning is a popular reinforcement learning algorithm.
This algorithm has however been studied and analysed mainly in the infinite horizon setting. There are several important applications which can be modeled in the framework of finite horizon Markov decision processes. We develop a version of Q-learning algorithm for finite horizon Markov decision processes (MDP) and provide a full proof of its stability and convergence. Our analysis of stability and convergence of finite horizon Q-learning is based entirely on the ordinary differential equations (O.D.E) method. We also demonstrate the performance of our algorithm on a setting of random MDP along with an application on smart Grids.
\end{abstract}

\section{Introduction}
\label{introduction}

Markov decision process (MDP) is a popular framework to study sequential decision making under uncertainty. An MDP is generally defined via a 5-tuple $\langle S,A,P,R,\beta \rangle $ where $S$ is the set of states , $A$ is the set of actions, $P$ is the transition probability matrix, $R$ is the reward function and $\beta$ is a scalar that can take the values $0 < \beta \leq 1$. In infinite horizon discounted reward problems, we require $\beta<1$. Quite often, in real-life applications, $P$ and $R$ are not available to us, however, we have access to data in the form
of several `state-action-reward-next state' tuples. Reinforcement learning algorithms learn the optimal policies and value functions from such data samples. Amongst the model free reinforcement learning algorithms, Q-learning is a popular one that has been widely studied both theoretically and over a range of applications. 

Now we give a brief survey of Q-learning-type algorithms in the literature. Q-learning has been extensively studied and many improvements have been proposed to the basic algorithm in
the literature. In Double Q-learning  \cite{doubleQlearning}, two estimators of Q-values are used to improve the empirical performance by reducing the over-estimations of Q-values.
Speedy Q-learning \cite{speedyQlearning} is another algorithm where bounds on the number of iterations needed for convergence have been improved.
Generalized speedy Q-learning \cite{GSQL} improves further upon Speedy Q-learning by adopting the technique of successive relaxation. The improvement occurs because
the contraction factor of the successive relaxation Bellman operator is lower in value than that of the standard Bellman operator. In  Zap Q-learning \cite{zapQlearning}, a two-timescale update rule for matrix gain is used which results in an improvement over standard Q-learning.

Even though there are many applications that are finite horizon in nature, 
previous works have mainly dealt with the infinite horizon setting. Even if the number of stages is reasonably large, using the infinite horizon setting as an approximation introduces errors
in the solutions. For instance, a stationary policy is invariably optimal in infinite horizon settings but rarely so when the setting is finite horizon in nature.
Further, finite horizon settings allow for non-stationary transition probabilities and reward functions, unlike infinite horizon MDPs. Most of the development in reinforcement learning has been towards the design of algorithms for infinite horizon problems. Finite horizon temporal difference learning has been studied recently in \cite{sutton-finite}. A finite horizon version of
 Q-learning has been studied in \cite{finiteHorizon1998}. However, a rigorous proof of convergence has not been shown in that reference. 
 We present the finite horizon Q-learning algorithm and provide a rigorous proof of convergence along with an experimental study on random MDPs.
The main contributions in this paper are as follows:
\begin{itemize}
    \item We present a Finite Horizon Q-learning algorithm for the general case where the transition dynamics and reward structure are also stage-dependent in addition to them being dependent on states and actions.
    \item We provide a stochastic approximation based analysis for the stability and convergence of Finite Horizon Q-learning. We show that the Finite horizon Q-learning recursions
    are both stable and converge to the set of optimal Q-values almost surely.
    \item We show the results of experiments on a setting of random MDPs and observe that our results conform with the theoretical results.
\end{itemize}

The rest of the paper is arranged as follows. In the next section, we present basic results in finite horizon reinforcement learning. We present here the Finite Horizon Dynamic Programming (FHDP)
algorithm and show that it gives the optimal Q-values. This is then followed in the next section by a description of our proposed algorithm -- the Finite Horizon Q learning. In the subsequent section, we present the complete proof of stability and convergence
of our finite horizon reinforcement learning algorithm. We then present the results of experiments conducted in the setting of random MDPs in the subsequent section. The final section presents conclusions and describes future work.

\section{Preliminaries}
\label{preliminaries}

Our basic setting involves a finite horizon MDP with horizon length $N<\infty$. Let $n=0,1,\ldots,N-1$ denote the $N$ stages of decision making with $N$ as the termination instant. We consider here a setting involving cost minimization as opposed to reward maximization as it appears to be more natural for optimal control problems. Let $g_n(i,a,j)$ (resp.~$p(i,a,j)$) denote the single-stage cost (resp.~transition probability) when the state at instant $n$ is $i$ and action chosen is $a$, and the next state (i.e., the one at instant $n+1$) is $j$. Further, let $g_N(i)$ denote the terminal cost when the terminating state is $i\in S$. Let $S$ and $A$ respectively denote the state and action spaces of the MDP. In particular, we let $A(i) \subseteq A$ be the set of feasible actions in state $i$. In general, one may let $S$ and $A$ be time-dependent sets. We however select them to be time-invariant for simplicity.

Let $\pi = \{\pi_0,\pi_1,\ldots,\pi_{N-1}\}$ represent a policy where $\pi_k(i) \in A(i)$, $\forall k=0,1,\ldots,N-1$. The idea is that when following policy $\pi$, at instant $k$, the action is chosen according to the function $\pi_k$, $k=0,1,\ldots,N-1$.
Let $J^\pi(i)$, $i\in S$, denote the long-term expected cost:
\[
J^\pi(i) = E_\pi\left[\sum_{k=0}^{N-1} g_k(s_k, a_k, s_{k+1}) + g_N(s_N) \mid s_0=i\right].
\]
Let $\Pi$ denote the set of all policies as above.
The goal then is to find a policy $\pi^* \in \Pi$ that gives the optimal long-term expected cost given by
\[
J^*(i) \stackrel{\triangle}{=} J^{\pi^*}(i) = \min_{\pi\in \Pi} J^\pi(i), \mbox{ } i\in S.
\]
Define Q-values $Q_\pi(i,a)$ as follows:
\[
Q_\pi(i,a) = E[ \sum_{k=0}^{N-1} g_k(s_k,\pi_k(s_k), s_{k+1} + g_N(s_N)\]
\[ \mid s_0=i, a_0=a].
\]
Here the initial action chosen in the initial state $i$ is $a$. Subsequent actions are chosen according to the policy $\pi$ that now kicks in from instant 1 onwards. 
Let
\[
Q^*(i,a) = \min_{\pi\in\Pi} Q_\pi(i,a),
\]
where $\Pi$ is the set of policies starting from instant 0. Note however that since the initial action is $a$ (in state $s$), the initial action is not according to $\pi_0$ in general.
Consider now the finite horizon Dynamic Programming (DP) algorithm \eqref{DP1}-\eqref{DP2}
in terms of the Q-function. 

\begin{equation}
\label{DP1}
    Q_N(s_N, a_N) = g_N(s_N),
\end{equation}
\[
Q_k(s_k,a_k) = E_{s_{k+1}}[g_k(s_k,a_k,s_{k+1})
\]
\begin{equation}
    \label{DP2}
   +  \min_{a_{k+1}\in A(s_{k+1})}
    Q_{k+1}(s_{k+1},a_{k+1})],
\end{equation}
$k=N-1,N-2,\ldots,0$.

\begin{proposition}
\label{DP-optimality}
For every initial state-action tuple $(s,a)$, the optimal Q-value $Q^*(s,a)$ of the basic problem equals $Q_0(s,a)$ given by \eqref{DP1}-\eqref{DP2} proceeding backwards in time from $k=N-1$ to $k=0$. Moreover, the policy $\pi^*=\{\pi^*_0,\pi^*_1,\ldots,\pi^*_{N-1}\}$ obtained by letting $\pi^*_k(s_k) = a^*_k$ in \eqref{DP2}, $\forall s_k\in S, k=0,1,\ldots,N-1$ is optimal.
\end{proposition}

\begin{proof}
The proof is based on an induction argument and follows along the lines of Proposition 1.3.1 of \cite{bertsekas1} that is however shown for the case of finite horizon state-value function (not the state-action value function or the Q-value function as here).
We give the details below for completeness. Let
\[
Q^*_k(s_k,a_k) = E_{s_{k+1}}[g_k(s_k,a_k,s_{k+1}) \]\[+ \min_{\pi^{k+1}} E_{s_{k+2},\ldots,s_N}
[g_N(s_N)
+ \sum_{k+1}^{N-1} g_i(s_i,\pi_i(s_i),s_{i+1})]],
\]
$k=0,1,\ldots,N-1$ and
\[
Q^*(N)(s_N,a_N) = g_N(s_N).
\]
Assume for some $k$ and all feasible $(s_{k+1},a_{k+1})$ tuples, $Q^*_{k+1}(s_{k+1},a_{k+1}) = Q_{k+1}(s_{k+1},a_{k+1})$. Then,
\begin{align*}
& Q^*_k(s_k,a_k) = E_{s_{k+1}}\big[g_k(s_k,a_k,s_{k+1}) \\ & + \min_{\pi_{k+1}, \pi_{k+2}}
E_{s_{k+2,\ldots,s_N}}[g_N(s_N) \\ & + g_{k+1}(s_{k+1},\pi_{k+1}(s_{k+1}),s_{k+2}) \\
& + \sum_{i=k+2}^{N-1} g_i(s_i,\pi_i(s_i),s_{i+1})]\big]
\end{align*}
\[
= E_{s_{k+1}}[g_k(s_k,a_k,s_{k+1})\]\[
+ \min_{\pi_{k+1}, \pi_{k+2}}
E_{s_{k+2}}[g_{k+1}(s_{k+1},\pi_{k+1}(s_{k+1}),s_{k+2})\]\[ + E_{s_{k+3},\ldots,s_N}[g_N(s_N)
+ \sum_{i=k+2}^{N-1} g_i(s_i,\pi_i(s_i),s_{i+1})]]
\]
The second term on the RHS above can be written as
\[
\min_{\pi_{k+1}}
E_{s_{k+2}}[g_{k+1}(s_{k+1},\pi_{k+1}(s_{k+1}),s_{k+2})\]\[ + \min_{\pi^{k+2}} E_{s_{k+3},\ldots,s_N}[g_N(s_N) 
 + \sum_{i=k+2}^{N-1} g_i(s_i,\pi_i(s_i),s_{i+1})]]
\]
\[
= \min_{a_{k+1} \in A(s_{k+1})} Q^*(s_{k+1},a_{k+1}).
\]
Thus,
\[
Q^*_k(s_k,a_k) = E_{s_{k+1}} [ g_k(s_k,a_k,s_{k+1}) \]\[
+ \min_{a_{k+1}\in A(s_{k+1})} Q^*_{k+1}(s_{k+1},a_{k+1})]
\]
\begin{align*}
=~ & E_{s_{k+1}} [ g_k(s_k,a_k,s_{k+1}) \\ 
& \hspace{2cm}+ \min_{a_{k+1}\in A(s_{k+1})} Q_{k+1}(s_{k+1},a_{k+1})]
\end{align*}
\[= Q_k(s_k,a_k),
\]
$\forall s_k\in S, a_k \in A(s_k)$.
This completes the induction step.

\end{proof}

\section{Finite Horizon Q learning Algorithm}
\label{FHQLA}

Note that the DP algorithm \eqref{DP1}-\eqref{DP2} is a numerical procedure that computes the optimal Q-values (as shown by Proposition~\ref{DP-optimality}). This however relies on the fact that knowledge of the transition probabilities is available. In most real-life situations, however, such information is not available and one only has available (as described earlier) data samples of `state-action-reward-next state' tuples. As with the regular infinite horizon Q-learning algorithm of \cite{QlearningWatkins}, we present a stochastic approximation version of the finite horizon DP algorithm in Q-values \eqref{DP1}-\eqref{DP2}.
We provide below the full details of the finite horizon Q-learning algorithm. In the algorithm below, we let the learning rate (or step-size) be 
${\displaystyle a(m) = \ceil*{\frac{1}{(m+1)/10}}}$. This choice was seen to perform well in experiments. Nonetheless a wide range of learning rate choices can be explored that satisfy the standard stochastic approximation conditions that we describe below. 

\begin{algorithm}[htbp]
\caption{Finite Horizon Q-Learning}\label{alg:Finite horizon Q-Learning}
\hspace*{\algorithmicindent} \textbf{Notation:}\\ 
\hspace*{\algorithmicindent} $Q_n^{m}(i,a)$: Q-value at state $i$, action $a$, stage $n$, \\ 
\hspace*{2.1cm} recursion $m$.

\hspace*{\algorithmicindent} $a(m)$: step-size at recursion index $m$\\
\hspace*{\algorithmicindent} $Q_N(i,a)$ : Q-value for state $i$ and action $a$ at \hspace*{1cm} \\ \hspace*{2.1cm} terminal stage ($N$). \\
\hspace*{\algorithmicindent} $g_n(i,a,j)$:  Single stage reward for stage $n$
where\\ 
\hspace*{\algorithmicindent} current  state  is $i$, action is $a$ and next state is $j$. 
\hspace*{\algorithmicindent} $g_N(i)$: Terminal reward at the $N^{th}$ stage when \\
\hspace*{1.6cm} terminal state is $i$. \\
\hspace*{\algorithmicindent} $A(j)$: Set of feasible actions in state $j$.

\hspace*{\algorithmicindent} $\eta(i,a)$: Sampling function taking input $(i,a)$ as \\ \hspace*{1.7cm} state-action pair and returns the next state.

\hspace*{\algorithmicindent} \textbf{Input:} Samples of the form\\ \hspace*{\algorithmicindent} $\big(i\ \text{(current state)},\ a\ \text{(action)},\ r\ \text{(reward)},\ j\ \text{(next state)}\big)$. 

\hspace*{\algorithmicindent} \textbf{Output:} Updated Q-value $Q_{n}^{m+1}(i,a)$ estimated after\\
\hspace*{2cm}$m$ iterations of the algorithm.

\hspace*{\algorithmicindent} \textbf{Initialization:}$\ Q_n^0(i,a) = 0$, $\forall (i,a), n=0,\ldots,N-1$, \hspace*{2.7cm} and $Q_N^0(i,a) =g_N(i), \forall (i,a)$

\begin{algorithmic}[1]
\Procedure{Finite Horizon Q-Learning:}{}
\State $a(m) = \ceil*{\frac{1}{(m+1)/10}} $
\State $j = \eta(i,a)$ (from samples)
\State $Q_n^{m+1}(i,a) = \big(1-a(m)\big)\Big( Q_n^{m}(i,a)\Big) + a(m) $ \\
\hspace*{0cm}$ \times \Big(g_n\big(i,a\big) + \underset{b \in A(j) }{\min}Q_{n+1}^{m}\big(j,b \big)\Big),
n=0,1,\ldots,N-1$,
\State $Q^{m+1}_N(i,a) = g_N(i)$, $\forall (i,a)$ tuples.
\State \textbf{return} $Q_{n}^{m+1}(i,a)$ 
\EndProcedure
\end{algorithmic}
\end{algorithm}


\section{Proof of Stability and Convergence}
\label{proof}

In this section, we give a proof of the stability and convergence of our finite horizon Q-learning algorithm. Our proof relies on verifying the Borkar and Meyn conditions for stability and convergence of general stochastic approximation algorithms, see \cite{borkarMeyn}. We first describe the conditions for stability and convergence as described in \cite{borkarMeyn} as well as the main results there. Later we shall show that our algorithm meets these conditions. 

Consider the following stochastic approximation recursion in $\mathcal{R}^d$:
\begin{equation} 
    X(m+1) = X(m) + a(m)(h(X(m)) + M(m+1)),  \label{SAeq}
\end{equation}
where $m \geq 0$, $X(m) = (X_1(m),\dots ,X_d(m))^T$, 
$h: \mathds{R}^d \rightarrow \mathds{R}^d$, $a(m)$ is a sequence of positive real numbers,
and $M(n), n\geq 0$ is a zero-mean noise sequence.

Consider now the O.D.E: 
\begin{align}
     \dot{x}(t) = h(x(t)). \label{basicODE}
\end{align}
The limit points of the stochastic recursion \eqref{SAeq} can be seen to be limit points of the O.D.E \eqref{basicODE}, see \cite{benaim}. The O.D.E approach to stochastic approximation (originally due to \cite{ljung} and \cite{kushcla}) is thus useful as one can understand the limit set of the stochastic recursion by analysing the limit set of the corresponding O.D.E. Let $x^*$ denote the unique globally asymptotically stable (UGAS) attractor for the O.D.E \eqref{basicODE}. Notice that while such an attractor may not always exist, in our setting, since we are analysing the finite horizon Q-learning algorithm, it will be seen that the UGAS attractor exists and just corresponds to the vector of optimal Q-values.

We now list the assumptions of \cite{borkarMeyn} followed by the main results from that paper summarized as Theorem 1 -- Theorem 2 below.
\begin{assumption}
\label{a1}
\begin{enumerate}
    \item 
The function $h:\mathbb{R}^d\rightarrow \mathbb{R}^d$ is Lipschitz continuous. 
\item The sequence of functions ${\displaystyle h_r(x) \stackrel{\triangle}{=} \frac{h(rx)}{r}}$ satisfy that $h_r(x) \rightarrow h_\infty(x)$, for some function
$h_\infty:\mathbb{R}^d\rightarrow \mathbb{R}^d$ uniformly on compacts.
\item The ODE 
\begin{equation}
    \dot{x}(t) = h_{\infty}(x(t))  \label{limitODE}
\end{equation}
has the origin in $\mathbb{R}^d$ as it's UGAS attractor.
\end{enumerate}
\end{assumption}

\begin{assumption}
\label{a2}
\begin{enumerate}
    \item 
The sequence $\lbrace M(n), \mathcal{F}_n, n\geq 0 \rbrace$
is a square-integrable martingale difference sequence where
$\mathcal{F}_n=\sigma(X(i),M(i),i\leq n)$ is an associated sequence of sigma fields. 
\item For some constant $C_0<\infty$ and any initial condition $X(0)\in\mathbb{R}^d,n \geq 0$, 
\begin{equation}
\label{sbound}
     E[\parallel M(n+1)\parallel^2 | \mathcal{F}_n ] \leq C_0(1+ \parallel X(n)\parallel^2),
\end{equation}
for all $n\geq 0$.
\end{enumerate}
\end{assumption}

\begin{assumption}
\label{a3}
The sequence $\{a(n)\}$ is a deterministic sequence of step-sizes that satisfy the following conditions:
\begin{enumerate}
    \item $a(n) >0$, $\forall n\geq 0$,
\item $\sum_n a(n) =\infty$,    
\item $\sum_n a(n)^2 <\infty$
\end{enumerate}
\end{assumption}

\begin{theorem}[Theorem 2.1 of \cite{borkarMeyn}] \label{thm1}
Under Assumptions~\ref{a1}--\ref{a3}, for any initial condition $X(0)\in R^d$,
$$\underset{\mathrm{n}}{\mathrm{sup}} \parallel X(n) \parallel < \infty \text{   almost surely (a.s).} $$
\end{theorem}

\begin{theorem}[Theorem 2.2 of \cite{borkarMeyn}] \label{thm2}
Assume Assumptions~\ref{a1}--\ref{a3} hold and that the ODE \eqref{basicODE} has a unique globally asymptotically stable equilibrium $x^*$. Then for any initial condition $X(0)\in R^d$, $$\lim_{n \to \infty} X(n) = x^*.$$
\end{theorem}

We now recall the update rule of the finite horizon Q-learning algorithm:
\[
Q_n^{m+1}(i,a) = Q_n^{m}(i,a) + a(m)(g_n(i,a, \eta_n^m(i,a)) \]
\begin{equation}
    \label{ql1}
+ \underset{b \in A(\eta_n^m(i,a)) }{\min}Q_{n+1}^{m}(\eta_n^m(i,a),b) - Q_n^m(i,a)),
\end{equation}
$n=0,1,\ldots,N-1$, $m\geq 0$,
\begin{equation}
\label{ql2}    
Q^{m+1}_N(i,a) = g_N(i), \mbox{ } \forall m\geq 0,
\end{equation}
$\forall (i,a)$. In the above $\eta_n^m(i,a)$ are i.i.d random variables having the distribution $p_n(i,a,\cdot)$. 
Let $\mathcal{G}_m = \sigma(Q_n^k(i,a), \eta_n^l(i,a), n=0,1,\ldots,N-1, k\leq m, l<m)$, $m\geq 0$ denote the sequence of associated sigma fields. Here we let $\eta_n^{-1}(i,a)=0$, $\forall (i,a)$ tuples.

Let $Q^m \stackrel{\triangle}{=} (Q_n^{m}(i,a), i\in S, a\in A(i), n=0,1,\ldots, N)^T$.
We rewrite \eqref{ql1}-\eqref{ql2} in the following unified form:
\[
Q_n^{m+1}(i,a) = Q_n^m(i,a) + a(m)(h_n(i,a,Q^m)\]\[ + M_{n}^{m+1}(i,a)),
n=0,1,\ldots,N,\] where for $n=0,1,\ldots,N-1$,
\[ h_n(i,a,Q^m) = \sum_{j} p_n(i,a,j)(g_n(i,a,j)\]\[ + \min_{b\in A(j)} Q_{n+1}^m(j,b))
- Q_n^m(i,a),
\]
\begin{align*}
&M_n^{m+1}(i,a) = \\
&g_n(i,a,\eta_n^m(i,a)) + \min_{b\in A(\eta_n^m(i,a))} Q_{n+1}^m(\eta_n^m(i,a),b) \\
&-\sum_j p_n(i,a,j)(g_n(i,a,j) + \min_{b\in A(j)} Q_{n+1}^m(j,b)).
\end{align*}
Also, for $n=N$, we have 
\[h_N^m(i,a) = 0 \mbox{ and } M_n^{m+1}(i,a)=0.
\]
Let
$M^{m+1} \stackrel{\triangle}{=} (M_n^{m+1}(i,a), n=0,1,\ldots,N, a\in A(i), i\in S)^T$.
Also, let $h(Q) = (h_n(i,a,Q), a\in A(i),i\in S, n=0,1,\ldots,N)^T$.
The recursions \eqref{ql1}-\eqref{ql2} can thus together be written as
\[
Q^{m+1} = Q^m + a(m)(h(Q^m) + M^{m+1}), \mbox{ } m\geq 0.
\]

In what follows, we shall use $\parallel\cdot\parallel$ to denote the sup or the max norm.
We now proceed by verifying Assumptions~\ref{a1}-\ref{a3}. 
\begin{proposition}
\label{prop1-a1}
The functions $h$, $h_r$ (defined as in Assumption~\ref{a1}), $\forall r\geq 1$, and $h_\infty$ are Lipschitz continuous.
\end{proposition}

\begin{proof}
We show first that for two functions $Q$ and $Q'$,
\begin{align*}
& |\min_{a\in A(i)}Q(i,a) - \min_{a\in A(i)} Q'(i,a)| \\
& \hspace{1.5cm} \leq \max_{a\in A(i)}|Q(i,a)-Q'(i,a)|,
\end{align*}

for all $(i,a)$ tuples. Note that given functions $f$ and $g$ and a set $A$,
\[\inf_{x\in A} (f(x)+g(x)) = \inf_{x\in A, x=y} (f(x)+g(y))
\]
\[
\geq \inf_{x,y\in A}(f(x)+g(y)) = \inf_{x\in A} f(x) + \inf_{y\in B} g(y).
\]
Using $f-g$ in place of $f$, one obtains
\[
\inf_{x\in A} ((f-g)(x) +g(x)) \geq \inf_{x\in A} (f-g)(x) + \inf_{x\in A}g(x), \mbox{ or}
\]
\[
\inf_{x\in A} (f(x)-g(x)) \leq \inf_{x\in A} f(x) - \inf_{x\in A} g(x).
\]
Let $e(x) =-g(x), \forall x$. Then
\[
\inf_{x\in A}(f(x)+e(x)) \leq \inf_{x\in A} f(x) + \sup_{x\in A} e(x), \mbox{ or }
\]
\[
\inf_{x\in A}(f(x)+e(x)) - \inf_{x\in A} f(x) \leq \sup_{x\in A} e(x).
\]
Again with $e(x) = g(x)-f(x)$, we have
\[
\inf_{x\in A} g(x) - \inf_{x\in A} f(x) \leq |\sup_{x\in A}(g(x)-f(x))|.
\]
We now claim that 
\[
|\sup_{x\in A}(g(x)-f(x))| \leq \sup_{x\in A} |g(x)-f(x)|.
\]
Consider first the case when $\sup_{x\in A} (g(x)-f(x))\geq 0$. Then
\[
\sup_{x\in A}(g(x)-f(x)) \leq \sup_{x\in A} |g(x)-f(x)|.
\]
Now consider the case when $\sup_{x\in A} (g(x)-f(x))<0$. In this case,
$|g(x)-f(x)| = -(g(x)-f(x))$, $\forall x$. Then,
\[
|\sup_{x\in A} (g(x)-f(x))| = -\sup_{x\in A}(g(x)-f(x))
\]
\[
= \inf_{x\in A} (-(g(x)-f(x)) = \inf_{x\in A} |g(x)-f(x)| 
\]
\[\leq \sup_{x\in A}|g(x)-f(x)|.
\]
It follows that 
\[\inf_{x\in A} g(x)-\inf_{x\in A}f(x) \leq \sup_{x\in A}|g(x)-f(x)|.
\]
One may similarly show that
\[
\inf_{x\in A} f(x)-\inf_{x\in A}g(x) \leq \sup_{x\in A}|g(x)-f(x)|.
\]
The two inequalities above then imply
\[
|\inf_{x\in A} g(x) - \inf_{x\in A} f(x)| \leq \sup_{x\in A}|g(x)-f(x)|.
\]
The claim follows upon substituting $g(x)$ with $Q(i,a)$, $f(x)$ with $Q'(i,a)$, $A$ with $A(i)$ and noting that the sets $S$ and $A(i)$ for all $i$ are finite, hence one replaces the $\inf$ with $\min$ and $\sup$ with $\max$ operators. It is now easy to see that $h$ and $h_r$ are Lipschitz continuous for all $r\geq 1$. Now define $h_\infty(Q)=
(h_{n,\infty}(i,a,Q), i\in S, a\in A(i),n=0,1,\ldots,N)^T$, where
\[h_{n\infty}(i,a,Q) = \lim_{r\rightarrow\infty} \frac{h_n(i,a,rQ}{r}
\]
\[= \sum_{j}p_n(i,a,j)\min_{b\in A(j)} Q(j,b) - Q(i,a).
\]
It is again clear from the foregoing that $h_\infty(Q)$ is Lipschitz continuous as well.
\end{proof}

Let $Q^* \stackrel{\triangle}{=} (Q^*_n(i,a), i\in S, a\in A(i), n=0,1,\ldots,N)^T$ denote the vector of optimal Q-values. By Proposition~\ref{DP-optimality}, we observe the DP algorithm \eqref{DP1}-\eqref{DP2} gives the optimal $Q^*$.
\begin{lemma}
\label{lem-a1}
The following hold:
\begin{enumerate}
    \item The ODE $\dot{Q} = h(Q)$ has $Q^*$ as it's unique globally asymptotically stable equilibrium.
    \item The ODE $\dot{Q}=h_\infty(Q)$ has the origin as it's unique globally asymptotically stable equilibrium.
\end{enumerate}
\end{lemma}

\begin{proof}
\begin{enumerate}
    \item Note that the equilibria of the ODE $\dot{Q}=h(Q)$ correspond to $H\stackrel{\triangle}{=} \{Q\mid h(Q)=0\}$. However, it is easy to see that the system of equations that we obtain from $h(Q)=0$ correspond to the solution of the DP algorithm \eqref{DP1}-\eqref{DP2} and which is precisely $Q^*$. Moreover, $Q^*$ is unique.
    A similar argument as in \cite{avgcostABB} shows that $Q^*$ is an asymptotically stable attractor of $\dot{Q}=h(Q)$.
    \item It is again easy to see that $Q=0$ (the vector of all zeros or the origin) is an equilibrium of the ODE $\dot{Q} =h_\infty(Q)$. The fact that it is unique can be seen again from the DP algorithm \eqref{DP1}-\eqref{DP2} for the special case when all single-stage costs $g_k(i,a,j)$, $k=0,1,\ldots,N-1$ and $g_N(i)$ are zero (for all $i,j\in S, a\in A(i)$). Again a similar argument as \cite{avgcostABB} shows that the origin is an asymptotically stable attractor of $\dot{Q}=h_\infty(Q)$.
\end{enumerate}
\end{proof}

We now have the following result on the noise sequence $M^{m+1}$, $m\geq 0$.

\begin{proposition}
\label{prop-mart}
\begin{enumerate}
    \item The sequence $(M^{m+1},\mathcal{G}_m)$, $m\geq 0$ is a square-integrable martingale difference sequence.
    \item For some constant $C_1>0$, we have
    \[ E[\parallel M^{m+1} \parallel^2\mid \mathcal{G}_m] \leq C_1(1+\parallel Q^m\parallel^2)
    \]
\end{enumerate}
\end{proposition}
\begin{proof}
\begin{enumerate}
    \item 
Recall that 
$M^{m+1} \stackrel{\triangle}{=} (M_n^{m+1}(i,a), n=0,1,\ldots,N, a\in A(i), i\in S)^T$, where for $n=0,1,\ldots,N-1$,
\begin{align*}
& M_n^{m+1}(i,a) = \\ 
& g_n(i,a,\eta_n^m(i,a)) + \min_{b\in A(\eta_n^m(i,a))} Q_{n+1}^m(\eta_n^m(i,a),b) \\
& -\sum_j p_n(i,a,j)(g_n(i,a,j) + \min_{b\in A(j)} Q_{n+1}^m(j,b)).
\end{align*}
and $M_N^{m+1}=0$, $\forall m\geq 0$.  
It is easy to see that $M_n^{m}$ is $\mathcal{G}_m$-measurable $\forall m\geq 0$. Also, $M_n^m$, $m\geq 0$ are integrable random variables since $g_n, n=0,1,\ldots,N$ are all uniformly bounded. Also, since $Q^m_n$ are updated according to \eqref{ql1}-\eqref{ql2} from a given $Q^0 \in \mathcal{R}^d$, it can be seen that one can find a constant $K^m_n <\infty$ (uniformly over all sample paths), $\parallel Q^m_n\parallel \leq K^m_n <\infty$, $\forall m\geq 0, n=0,1,\ldots, N$. Thus, $M_n^m$ are all individually integrable. It is also easy to see that $M^m$, $m\geq 0$ are square-integrable random variables.
Further,
\begin{align*}
&E[g_n(i,a,\eta_n^m(i,a)) \\ 
& \hspace{1.7cm} + \min_{b\in A(\eta_n^m(i,a))} Q_{n+1}^m(\eta_n^m(i,a),b)
\mid \mathcal{G}_m] \\
& =\sum_j p_n(i,a,j)(g_n(i,a,j) + \min_{b\in A(j)} Q_{n+1}^m(j,b)).
\end{align*}
Thus, $E[M_n^{m+1}\mid \mathcal{G}_m]=0$, $\forall m=0,1,\ldots,N-1$. Further, $E[M_N^{m+1}\mid \mathcal{G}_m]=0$ since $M_M^{m+1}=0$, $\forall m\geq 0$.
\item Note that 
\begin{align*}
&E[|g_n(i,a,\eta_n^m(i,a)) \\
& \hspace{1.5cm} + \min_{b\in A(\eta_n^m(i,a))} Q_{n+1}^m(\eta_n^m(i,a),b)|^2
\mid \mathcal{G}_m]
\end{align*}
\[
\leq 2(|g_n(i,a,\eta_n^m(i,a)|^2 + |Q_{n+1}^m(\eta_n^m(i,a),b)|^2,
\]
$\forall b\in A(\eta_n^m(i,a))$. The claim follows from the fact that the single-stage costs $g_n, n=0,1,\ldots,N$ are all uniformly bounded.
\end{enumerate}
\end{proof}

We finally have the main result on stability and convergence.
\begin{theorem}
\label{main-result}
The iterates $Q^m$, $m\geq 0$, given by the Q-learning algorithm 
\eqref{ql1}-\eqref{ql2} satisfy (a) $\sup_m \parallel Q^m\parallel <\infty$ almost surely, and (b) $Q^m \rightarrow Q^*$ as $m\rightarrow\infty$ almost surely.
\end{theorem}
\begin{proof}
Assumption~\ref{a1} is shown as a consequence of Proposition~\ref{prop1-a1} and Lemma~\ref{lem-a1}. Proposition~\ref{prop-mart} shows that Assumption~\ref{a2} holds.
Finally, Assumption~\ref{a3} can be seen to easily hold for our choice of step-sizes, viz.,
${\displaystyle a(m) = \ceil*{\frac{1}{(m+1)/10}}}$. The claim now follows from Theorem~\ref{thm1}--Theorem~\ref{thm2}.
\end{proof}

\section{Experiments and Results}
\label{experiments}

We implemented our
Finite Horizon Q-learning algorithm on the setting of random MDPs for different combinations of the following triplet $(N, |S|, |A|)$ of number of stages, number of states and number of actions, respectively. The algorithm was terminated in each experiment for the following termination condition:
\(\parallel Q_{prev} - Q_{curr} \parallel\  \leq \ \epsilon, \mbox{ where } \)
\begin{itemize}
    \item $Q_{prev}$: Vector of Q-values of size $N\times |S|\times |A|$ obtained at the previous iteration
    \item $Q_{curr}$: Vector of Q-values of size $N\times |S|\times |A|$ obtained at the current iteration
    \item $\epsilon$: a small tolerance or threshold level (set to $0.05$ in the experiments)
\end{itemize} 
The various quantities used in the following are explained below.
\begin{itemize}
\item $Q^*$: Final value of the Q-value function (viewed as a vector over $N\times S\times A$) as obtained from the Finite Horizon Q-learning Algorithm.
\item $Q^{DP}$: Final value of the Q-value function (viewed as a vector over $N\times S\times A$) as computed by the finite horizon dynamic programming algorithm \eqref{DP1}-\eqref{DP2}.
\item \text{Error} $=\parallel Q^* - Q^{DP} \parallel$ is the norm difference between $Q^*$ and $Q^{DP}$. 
    \item $J_n^*$: Value function at stage $n$ (having $|S|$ components) as computed by the algorithm.
    \item $\pi_n^*$: Optimal policy function at stage $n$, having $|S|$ components.
\end{itemize}

We mention here that since the experimental setting was of random MDPs, we could make use of the transition probability information for running the DP algorithm \eqref{DP1}-\eqref{DP2} and estimate the error across various runs when using our Finite Horizon Q-learning algorithm. 

\begin{table}[H]
\begin{center}
\begin{tabular}{|c|c|c|c|}
\hline
\textbf{$(N, |S|, |A|)$} & $\boldsymbol{\epsilon}$ & \textbf{Error}  & \textbf{\begin{tabular}[c]{@{}c@{}}Number of \\ Iterations\end{tabular}}\\ \hline
$(20, 50, 10)$ & 0.05 & 6.8720 & 82589 \\ \hline
$(20, 20, 10)$ & 0.05 & 3.5319 & 38268 \\ \hline
$(10, 20, 10)$ & 0.05 & 2.1433 & 28603 \\ \hline
$(10, 5, 5)$   & 0.05 & 1.2491 &  7507 \\ \hline
\end{tabular}
\end{center}
\caption{Algorithm performance in different settings.}
\label{summary}
\end{table}

For the four different settings of $(N,|S|,|A|)$ as shown in Table \ref{summary} , seven plots are provided for each of the settings. These individually correspond to the value function  and policy obtained upon termination of the algorithm for the first stage, the middle stage and penultimate stage as well as error as a function of number of iterations (see above) as obtained when using the finite horizon Q-learning algorithm.
Table \ref{summary} summarizes the performance in terms of the Error metric and the number of iterations needed in each of the four settings for the algorithm in order for norm-difference between the Q-values in the current and previous iterates of the algorithm to fall below $\epsilon=0.05$. From the plots we see that (a) the error in each case diminishes as the number of iterates increases, (b) the value function as obtained from the Finite Horizon Q-learning algorithm closely resembles the value function obtained using the DP algorithm \eqref{DP1}-\eqref{DP2} and (c) the plot of the policy function at different stages indicates that in general the best actions for the various states are quite different. 

\begin{figure}
     \centering
     \begin{subfigure}[b]{0.4\textwidth}
         \centering
         \includegraphics[width=\textwidth]{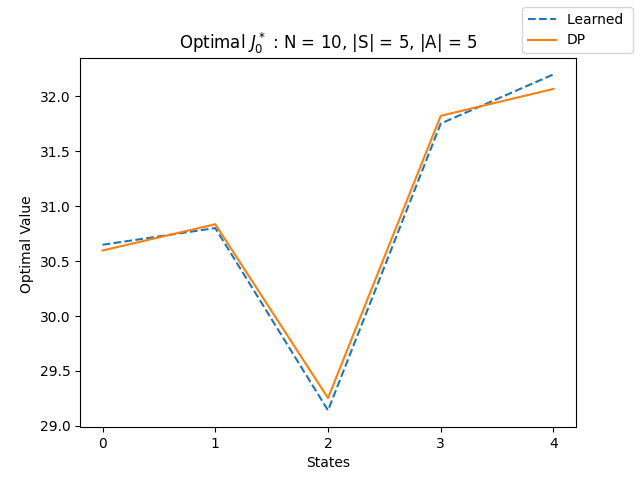}
         \caption{Estimated Value Function}
         \label{fig:value_10_5_5_h0}
     \end{subfigure}
     \hfill
     \begin{subfigure}[b]{0.4\textwidth}
         \centering
         \includegraphics[width=\textwidth]{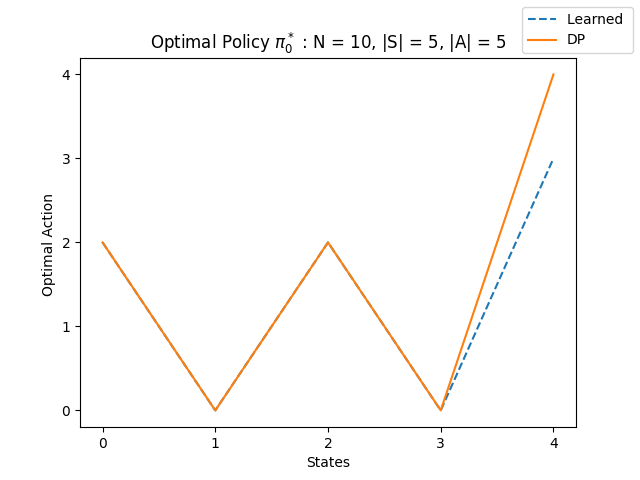}
         \caption{Estimated Policy}
         \label{fig:policy_10_5_5_h0}
     \end{subfigure}
     \caption{Performance for Stage 0 of (10, 5, 5) Setting}
\end{figure}

\begin{figure}
     \centering
     \begin{subfigure}[b]{0.4\textwidth}
         \centering
        \includegraphics[width=\textwidth]{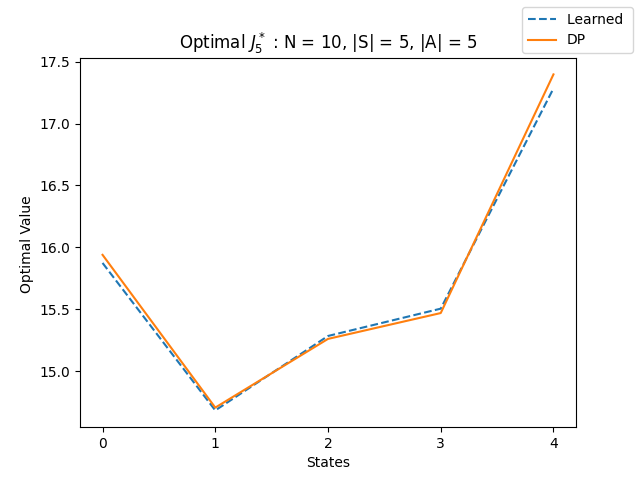}
        \caption{Estimated Value Function}
         \label{fig:value_10_5_5_h5}
    \end{subfigure}
     \hfill
     \begin{subfigure}[b]{0.4\textwidth}
         \centering
        \includegraphics[width=\textwidth]{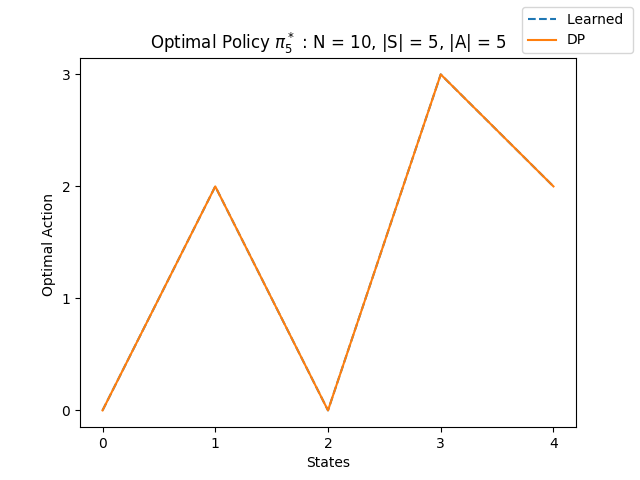}
         \caption{Estimated Policy}
         \label{fig:policy_10_5_5_h5}
    \end{subfigure}
     \caption{Performance for Stage 5 of (10, 5, 5) Setting}
\end{figure}

\begin{figure}
     \centering
     \begin{subfigure}[b]{0.4\textwidth}
         \centering
         \includegraphics[width=\textwidth]{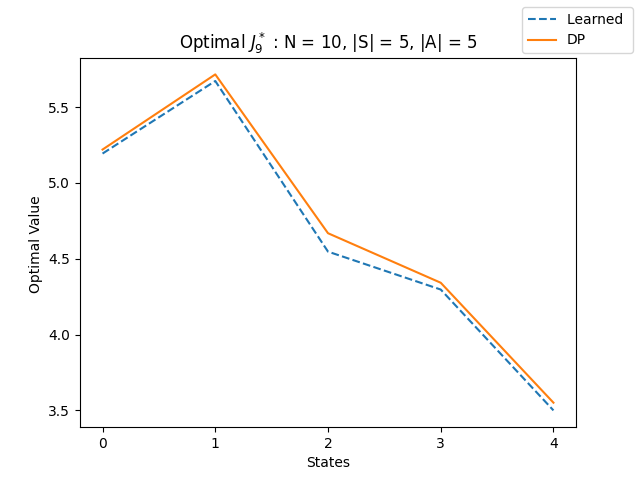}
         \caption{Estimated Value Function}
         \label{fig:value_10_5_5_h9}
     \end{subfigure}
     \hfill
     \begin{subfigure}[b]{0.4\textwidth}
         \centering
         \includegraphics[width=\textwidth]{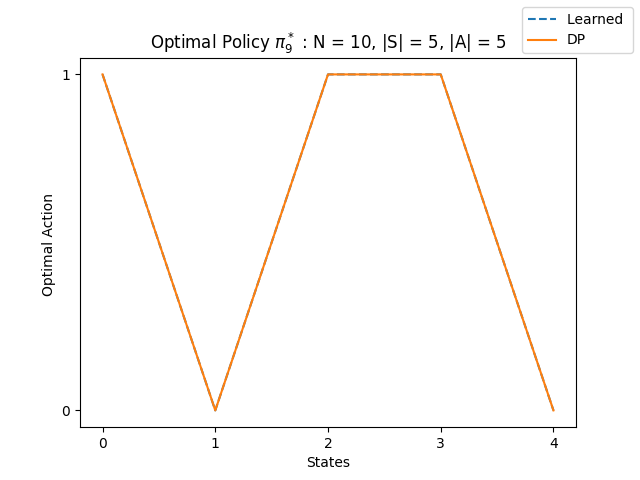}
         \caption{Estimated Policy}
         \label{fig:policy_10_5_5_h9}
     \end{subfigure}
     \caption{Performance for Stage 9 of (10, 5, 5) Setting}
\end{figure}


 
\begin{figure}
     \centering
     \begin{subfigure}[b]{0.4\textwidth}
         \centering
         \includegraphics[width=\textwidth]{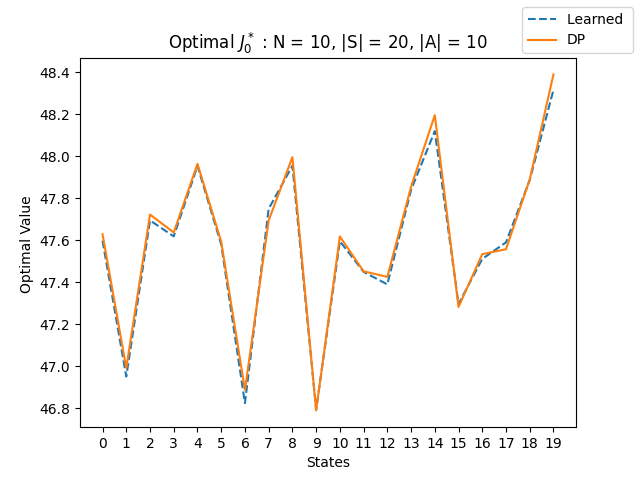}
         \caption{Estimated Value Function}
         \label{fig:value_10_20_10_h0}
     \end{subfigure}
     \hfill
     \begin{subfigure}[b]{0.4\textwidth}
         \centering
         \includegraphics[width=\textwidth]{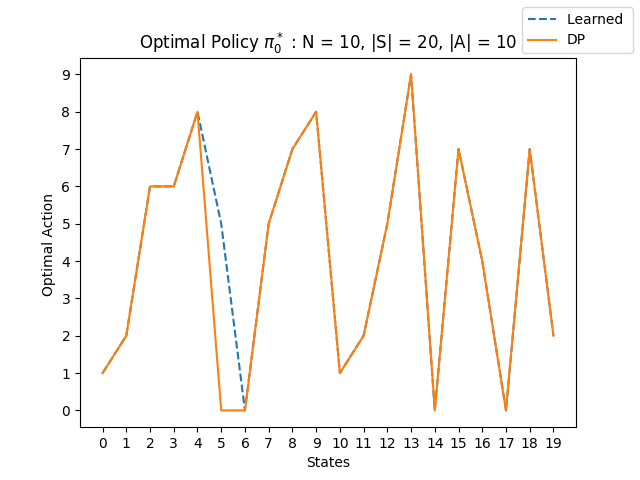}
         \caption{Estimated Policy}
         \label{fig:policy_10_20_10_h0}
     \end{subfigure}
     \caption{Performance for Stage 0 of (10, 20, 10) Setting}
\end{figure}

\begin{figure}
    \centering
    \begin{subfigure}[b]{0.4\textwidth}
        \centering
        \includegraphics[width=\textwidth]{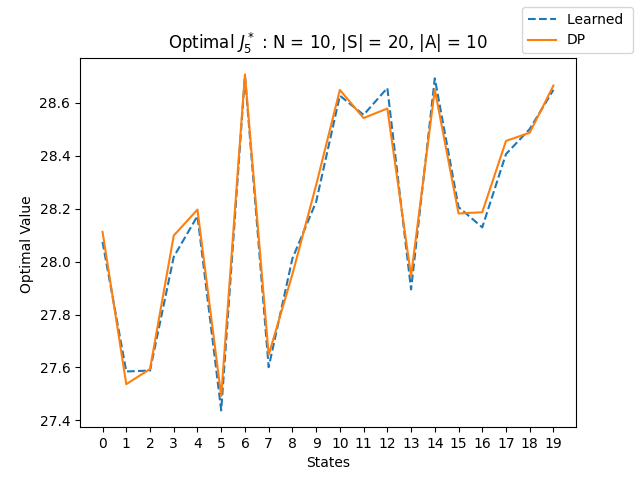}
        \caption{Estimated Value Function}
        \label{fig:value_10_20_10_h5}
    \end{subfigure}
    \hfill
    \begin{subfigure}[b]{0.4\textwidth}
        \centering
        \includegraphics[width=\textwidth]{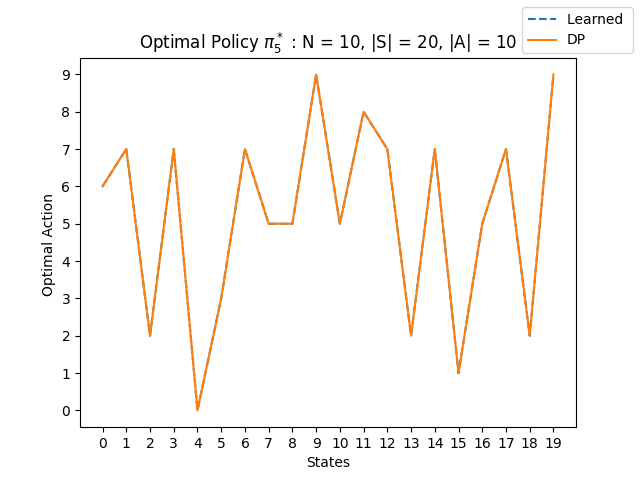}
        \caption{Estimated Policy}
        \label{fig:policy_10_20_10_h5}
    \end{subfigure}
    \caption{Performance for Stage 5 of (10, 20, 10) Setting}
\end{figure}

\begin{figure}
     \centering
     \begin{subfigure}[b]{0.4\textwidth}
         \centering
         \includegraphics[width=\textwidth]{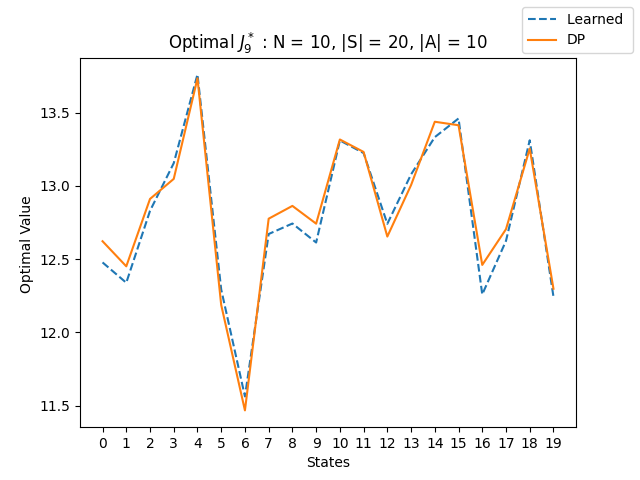}
         \caption{Estimated Value Function}
         \label{fig:value_10_20_10_h9}
     \end{subfigure}
     \hfill
     \begin{subfigure}[b]{0.4\textwidth}
         \centering
         \includegraphics[width=\textwidth]{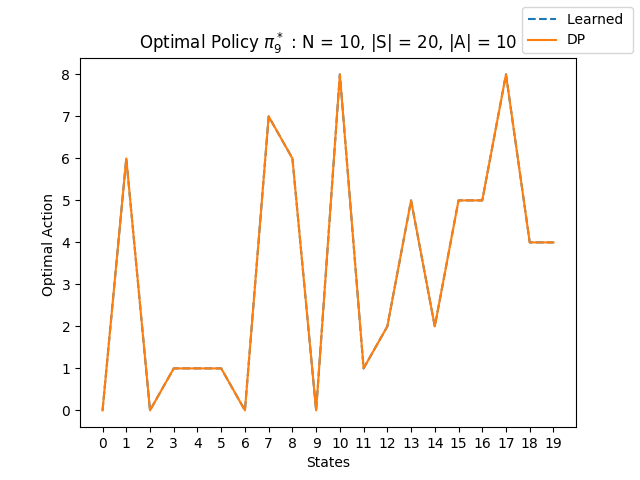}
         \caption{Estimated Policy}
         \label{fig:policy_10_20_10_h9}
     \end{subfigure}
     \caption{Performance for Stage 9 of (10, 20, 10) Setting}
\end{figure}

\begin{figure}
     \centering
     \begin{subfigure}[b]{0.4\textwidth}
         \centering
         \includegraphics[width=\textwidth]{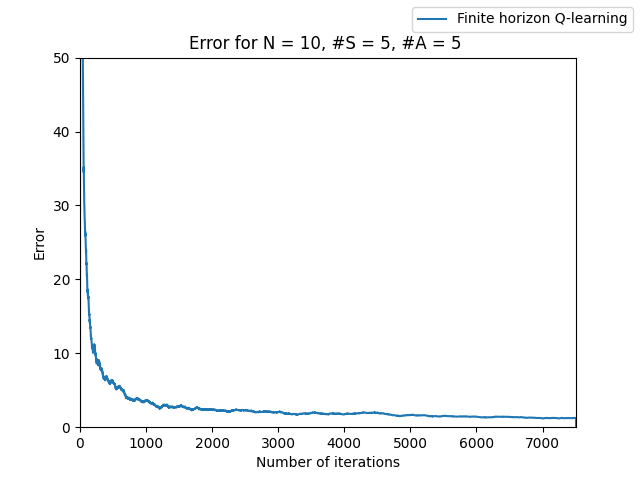}
         \caption{Error Plot for (10, 5, 5) Setting}
         \label{fig:error_10_5_5}
     \end{subfigure}
     \hfill
     \begin{subfigure}[b]{0.4\textwidth}
         \centering
         \includegraphics[width=\textwidth]{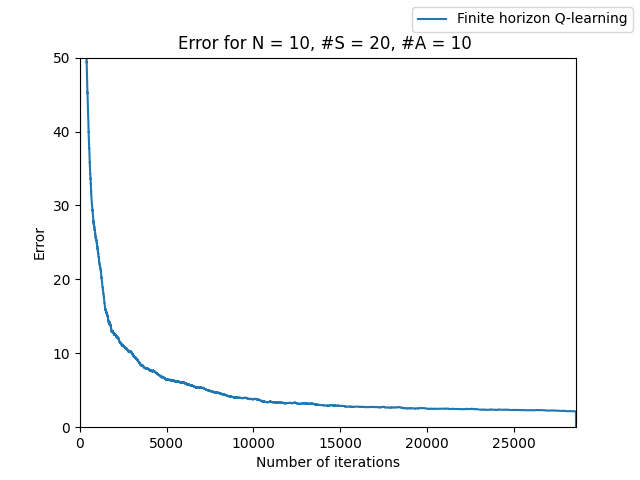}
         \caption{Error Plot for (10, 20, 10) Setting}
         \label{fig:error_10_20_10}
     \end{subfigure}
     \caption{Error Plots for Settings (10, 5, 5) and (10, 20, 10)}
\end{figure}


\begin{figure}
     \centering
     \begin{subfigure}[b]{0.4\textwidth}
         \centering
         \includegraphics[width=\textwidth]{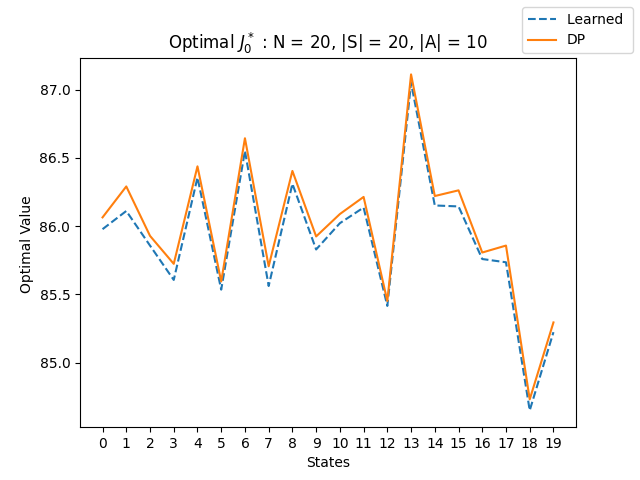}
         \caption{Estimated Value Function}
         \label{fig:value_20_20_10_h0}
     \end{subfigure}
     \hfill
     \begin{subfigure}[b]{0.4\textwidth}
         \centering
         \includegraphics[width=\textwidth]{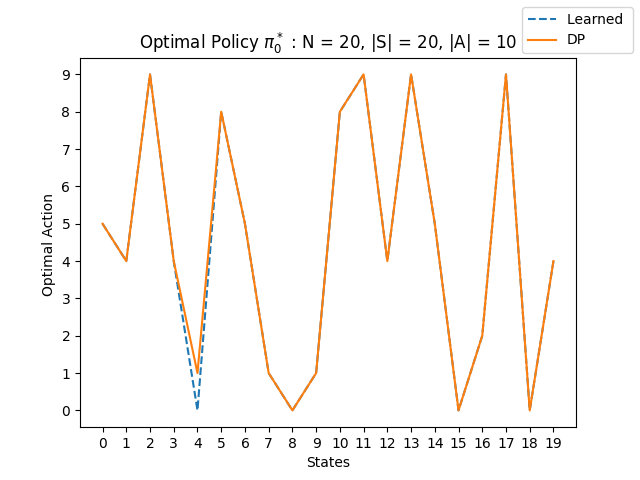}
         \caption{Estimated Policy}
         \label{fig:policy_20_20_10_h0}
     \end{subfigure}
     \caption{Performance for Stage 0 of (20, 20, 10)}
\end{figure}

\begin{figure}
    \centering
    \begin{subfigure}[b]{0.4\textwidth}
        \centering
        \includegraphics[width=\textwidth]{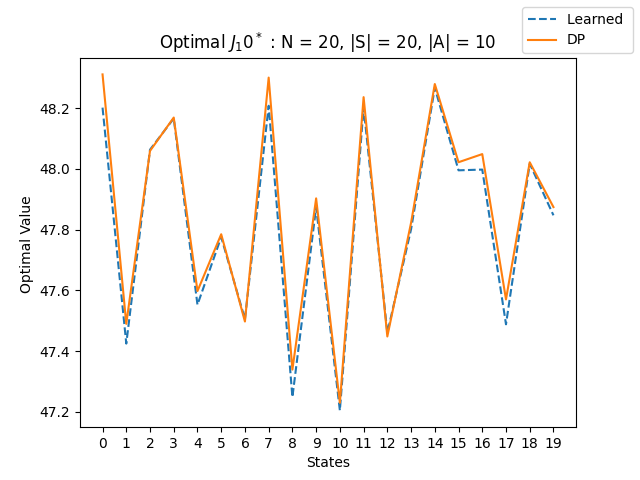}
        \caption{Estimated Value Function}
        \label{fig:value_20_20_10_h5}
    \end{subfigure}
    \hfill
    \begin{subfigure}[b]{0.4\textwidth}
        \centering
        \includegraphics[width=\textwidth]{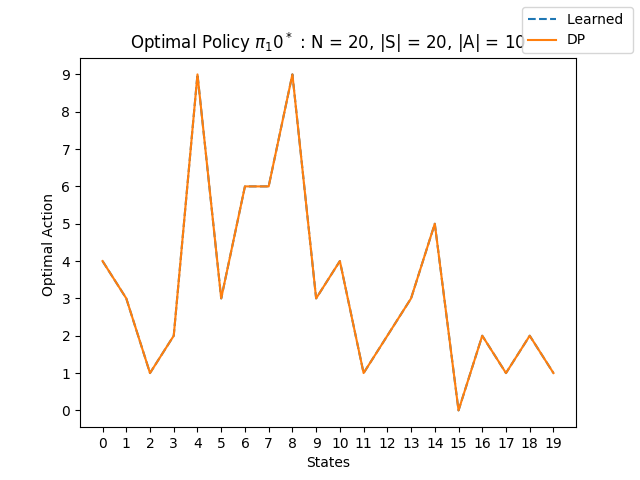}
        \caption{Estimated Policy}
        \label{fig:policy_20_20_10_h5}
    \end{subfigure}
    \caption{Performance for Stage 10 of (20, 20, 10) Setting}
\end{figure}

\begin{figure}
     \centering
     \begin{subfigure}[b]{0.4\textwidth}
         \centering
         \includegraphics[width=\textwidth]{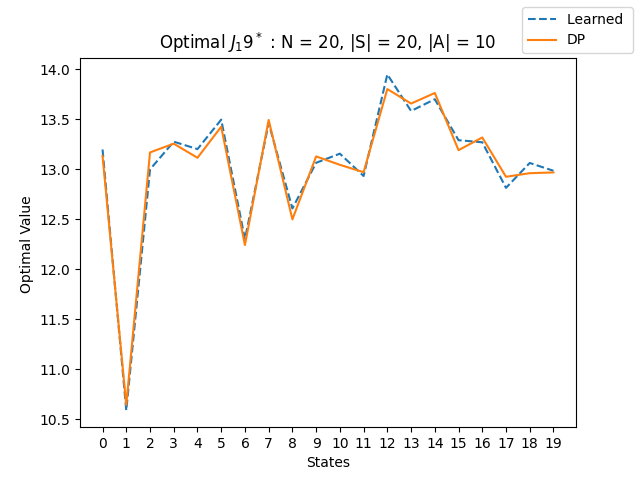}
         \caption{Estimated Value Function}
         \label{fig:value_20_20_10_h9}
     \end{subfigure}
     \hfill
     \begin{subfigure}[b]{0.4\textwidth}
         \centering
         \includegraphics[width=\textwidth]{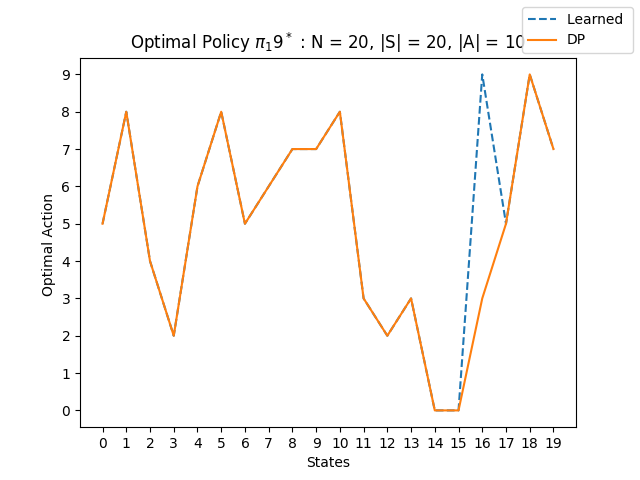}
         \caption{Estimated Policy}
         \label{fig:policy_20_20_10_h9}
     \end{subfigure}
     \caption{Performance for Stage 19 of (20, 20, 10)}
\end{figure}



\begin{figure}[htbp]
     \centering
     \begin{subfigure}[b]{0.4\textwidth}
         \centering
         \includegraphics[width=\textwidth]{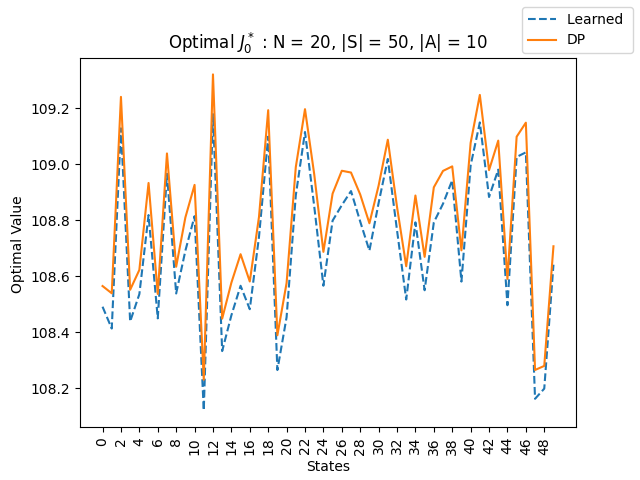}
         \caption{Estimated Value Function}
         \label{fig:value_20_50_10_h0}
     \end{subfigure}
     \hfill
     \begin{subfigure}[b]{0.4\textwidth}
         \centering
         \includegraphics[width=\textwidth]{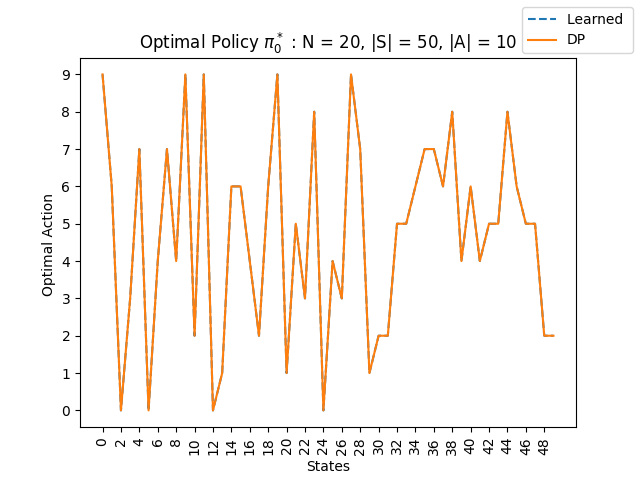}
         \caption{Estimated Policy}
         \label{fig:policy_20_50_10_h0}
     \end{subfigure}
     \caption{Performance for Stage 0 of (20, 50, 10)}
\end{figure}

\begin{figure}[htbp]
     \centering
     \begin{subfigure}[b]{0.4\textwidth}
         \centering
         \includegraphics[width=\textwidth]{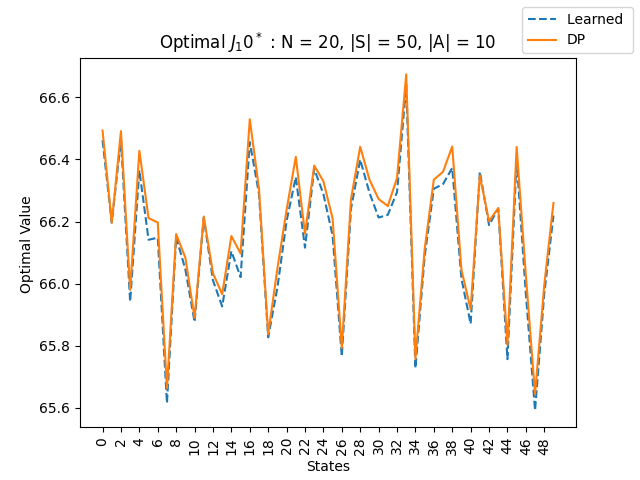}
         \caption{Estimated Value Function}
         \label{fig:value_20_50_10_h10}
     \end{subfigure}
     \hfill
     \begin{subfigure}[b]{0.4\textwidth}
         \centering
         \includegraphics[width=\textwidth]{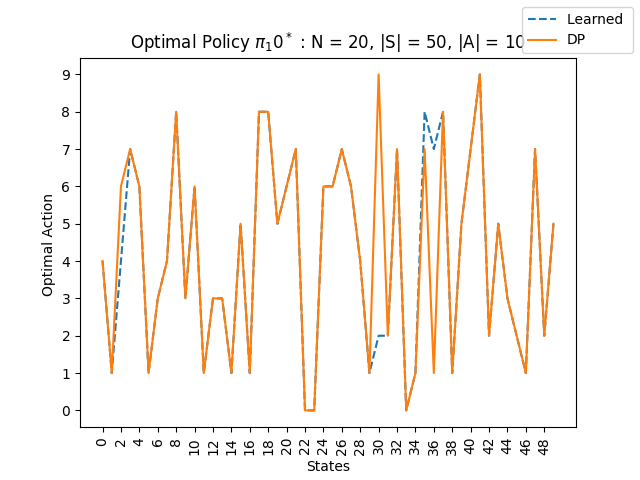}
         \caption{Estimated Policy}
         \label{fig:policy_20_50_10_h10}
     \end{subfigure}
     \caption{Performance for Stage 10 of (20, 50, 10) Setting}
\end{figure}

\begin{figure}[htbp]
     \centering
     \begin{subfigure}[b]{0.4\textwidth}
         \centering
         \includegraphics[width=\textwidth]{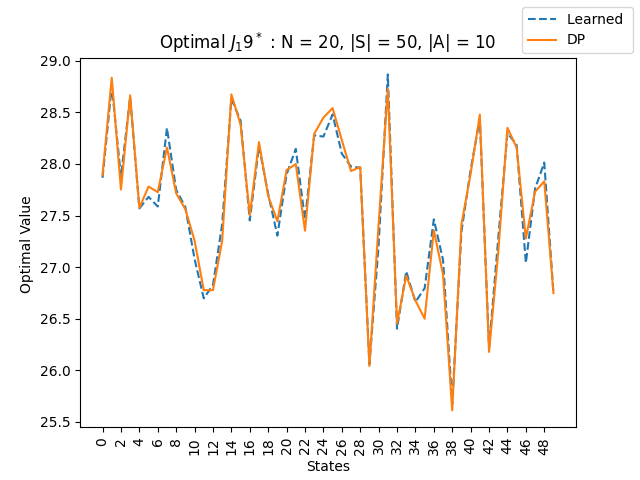}
         \caption{Estimated Value Function}
         \label{fig:value_20_20_10_h19}
     \end{subfigure}
     \hfill
     \begin{subfigure}[b]{0.4\textwidth}
         \centering
         \includegraphics[width=\textwidth]{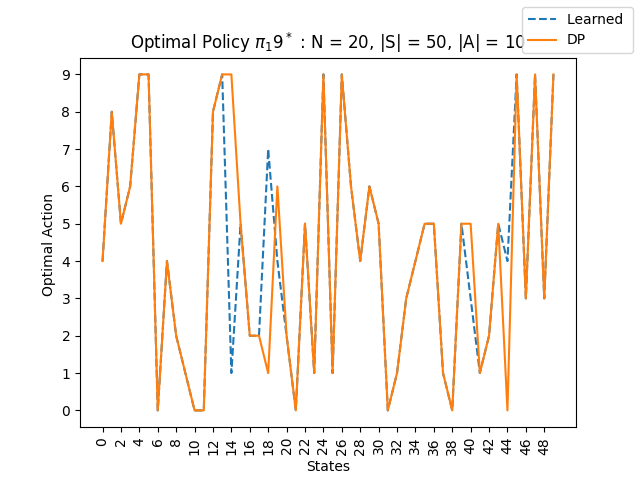}
         \caption{Estimated Policy}
         \label{fig:policy_20_50_10_h19}
     \end{subfigure}
     \caption{Performance for Stage 19 of (20, 50, 10)}
\end{figure}

\begin{figure}[htbp]
     \centering
     \begin{subfigure}[b]{0.4\textwidth}
         \centering
         \includegraphics[width=\textwidth]{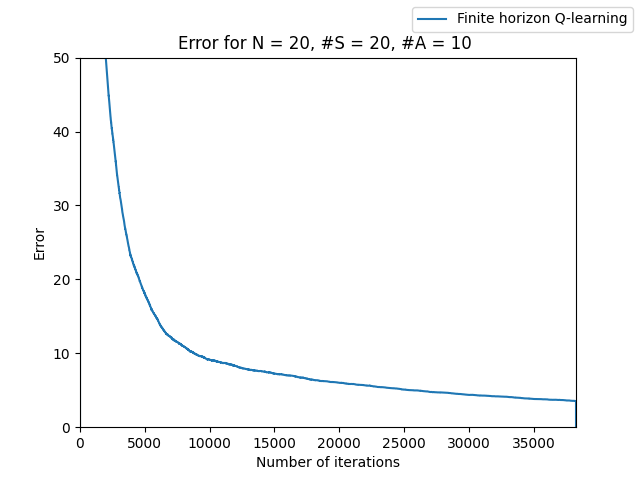}
         \caption{Error Plot for (20, 20, 10) Setting}
         \label{fig:error_20_20_10}
     \end{subfigure}
     \hfill
     \begin{subfigure}[b]{0.4\textwidth}
         \centering
         \includegraphics[width=\textwidth]{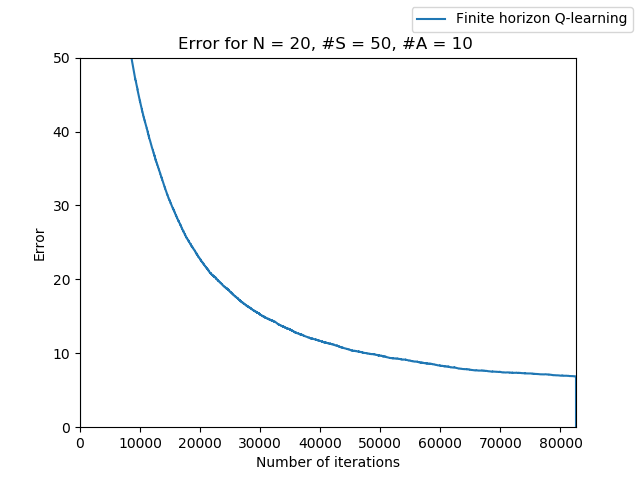}
         \caption{Error Plot for (20, 50, 10) Setting}
         \label{fig:error_20_50_10}
     \end{subfigure}
     \caption{Error Plots for Settings (20, 20, 10) and (20, 50, 10)}
\end{figure}


\section{Application on Smart Grid}
Electricity is a very important resource for the world to prosper. But its production and distribution involves a lot of challenges. Especially in countries with large geographical area, one find its distribution from a production centre to a far places challenging and very expensive.This can be solved partially with solutions like localized production from renewable sources such as solar. There may not be enough production sometimes, in that case one has to resort to buying from main grid. Also battery stores energy bought from main grid as well as other sources. Energy is spend from battery to meet the demand. These decisions have to be done in an intelligent manner so that overall expense is reduced. 
There are previous works on applying reinforcement learning techniques for
solving energy management is smart grids.  

Now we give a brief survey of various works on efficient energy management in smart grid using various approaches including reinforcement learning. Both demand side and supply side of the efficient energy management problem has been studied in \cite{unifiedRaghu} in a unified framework. Moreover this work extends to the flexibility of ADL (Activities of Daily Living) jobs in the same framework. This helps in reducing the peak load, which helps in reducing the infrastructure for production. The importance of smart grid for reducing the loss, to manage increasing demands and the overall massive transformation the smart grid will potentially bring in the energy sector has been discussed in \cite{farhangi}. One approach for managing pricing issue in distribution is that of \textit{broker agent} concept. This along with reinforcement learning is used to learn strategies for pricing in \cite{broker}. There are recent game theoretic approaches as well for smart grid management as in \cite{sGame}. This work allows microgrids to decide price contrary to the common approach where central main grid decides the price.
We formulate the problem of finding the optimal decision for the microgrid as a finite horizon MDP. We briefly discuss the overall process at a microgrid. Customers makes a demand of $d$ units in total to a microgrid. Microgrid already have a $b$ units  of power stored in its battery . Microgrid also knows the current price $p$, mostly dictated by maingrid. Microgrid also generates a $r$ units. These are the $4$ quantities that captures the process at microgrid at a time step. With these $4$ quantities microgrid should make decisions on two quantities, number of  units to buy from maingrid denoted by $u1$ and number of units to spent from battery ($u2$). Also total number of time steps or stages is denoted by $N$

\subsection{Terminology}
Following are the terms used:

\begin{itemize}
\item $d$ : demand from the customers end in number of units.
\item $b$ : number of units stored in the battery currently.
\item $p$ : current price of the power per unit.
\item $r$ : amount of renewal energy produced currently.
\item $u_1$ : action at main grid (number of units to be purchased from main grid)
\item $u_2$ : action at battery (number of units to be spent from battery)
\end{itemize}
 
\subsection{MDP formulation}
Now we will formulate this as a finite horizon MDP. For that we define the state, action and reward as follows:
\begin{itemize}
    \item State: $s = (d,b,p)$, a 3-tuple. 
    \item Action: $a = (u1, u2)$, an ordered pair.
    \item Reward: $g(s,a) = c*(d-u2) + p*u1$, c is a constant, c=1 in experiments.
    \item Probability: demand, price and renewal energy are coming from a probability distribution.
    \item Stage: $n$, current time stage, $n=0,1,\ldots,N-1$
    \item Terminal Stage reward: $g_N(s) = 0 , \forall s$
\end{itemize}
Now that we have all the parameters of the MDP we proceed towards finding the optimal policy using Finite Horizon Q-learning.

\subsection{Applying Q-learning}
Following are the updates upon applying $u1$ and $u2$ on the current state.
\begin{itemize}
    \item Demand update: next demand = drawn from a Markov chain.
    \item Battery update: 
    \begin{itemize}
        \item current battery = current battery + u1
        \item current battery = current battery - u2
        \item current battery = current bttery + r
        \item current battery = min ( current battery, maximum battery capacity)
        \item current battery = max (current battery, 0)
     \end{itemize}
    \item Price update: next price = drawn from a Markov chain
    \item Renewal energy update: 
\end{itemize}

These updates are applied in the finite horizon Q-learning algorithm. Average cost obtained are contrasted against two other basline algorithms as given below.\\

"fill demand" algorithm:\\
u1 : max(d-b, 0)\\
u2 : min(d, b)\\

"fill battery" algorithm: \\
u1: d + (maximum battery - b)\\ 
u2 : min(d,b)\\

Average cost for different scenarios is given in table below. We can see that Finite Horizon Q-learning outperforms baseline algorithms significantly. Comparing tables \ref{AlgoComparison1}  and \ref{AlgoComparison2} we also note that availability of renewal energy reduce average cost significantly, as expected.
\begin{table}[H]
\begin{center}
\begin{tabular}{|c|c|c|c|}
\hline
\textbf{$(h,d,b,p)$} & $FHQL$ & \textbf{$Fill~ Demand.$}  & \textbf{$Fill~ Battery.$}\\ \hline
$(10,4,4,4)$ & 2.357318   & 3.991445  & 5.662942  \\ \hline
$(20,4,4,4)$ & 2.493130   & 3.992053  & 5.659612  \\ \hline
$(20,5,5,5)$ & 2.931757   & 5.445172  & 8.141613 \\ \hline
\end{tabular}
\end{center}
\caption{Comparison of Algorithms without renewal energy (Average Cost)}
\label{AlgoComparison1}
\end{table}

\begin{table}[H]
\begin{center}
\begin{tabular}{|c|c|c|c|}
\hline
\textbf{$(h,d,b,p)$} & $FHQL$ & \textbf{$Fill~ Demand.$}  & \textbf{$Fill~ Battery.$}\\ \hline
$(10,4,4,4)$ & 0.115206   & 0.78571  & 5.66239  \\ \hline
$(20,4,4,4)$ & 0.274906   & 0.790913  & 5.656929 \\ \hline
$(20,5,5,5)$ & 0.439671   & 1.122509  & 8.138732 \\ \hline
\end{tabular}
\end{center}
\caption{Comparison of Algorithms with renewal energy (Average Cost)}
\label{AlgoComparison2}
\end{table}

\section{Conclusions and Future Work}
\label{conclusions}


We presented the Q-learning algorithm for finite horizon MDPs and gave a complete proof of stability and convergence to the set of optimal Q-values. Experiments indicate that the algorithm achieves almost the same Q-values as the dynamic programming algorithm that works with full model information. It will be interesting to develop RL algorithms with function approximation in the future. Experiments on more detailed settings should also be performed in the future.

\printbibliography

\end{document}